\documentclass{llncs}

\usepackage{cite}
\usepackage{multirow}
\usepackage{graphicx}
\usepackage{amsmath}
\usepackage{hyperref}
\usepackage{relsize}

\usepackage{amssymb}
\usepackage[utf8]{inputenc}
\usepackage{complexity}
\usepackage[algosection]{algorithm2e}
\usepackage{endnotes}
\usepackage{listings}
\usepackage{mathtools}
\usepackage{verbatim}
\usepackage{mathrsfs}
\usepackage{longtable}
\usepackage{enumitem}
\usepackage{etoolbox}
\usepackage{float}
\usepackage{color}

\begin{document}
	\title{New Techniques for Inferring L-Systems Using Genetic Algorithm\thanks{This research was supported in part by a grant from the Plant Phenotyping and Imaging Research Centre.}}
	\author{Jason Bernard \and Ian McQuillan}
	\institute{Department of Computer Science, University of Saskatchewan, Saskatoon, Canada 
		\email{jason.bernard@usask.ca}, 
		\email{mcquillan@cs.usask.ca}}
	
	\maketitle
	
	\begin{abstract}
		Lindenmayer systems (L-systems) are a formal grammar system that iteratively rewrites all symbols of a string, in parallel. When visualized with a graphical interpretation, the images have self-similar shapes that appear frequently in nature, and they have been particularly successful as a concise, reusable technique for simulating plants. The L-system inference problem is to find an L-system to simulate a given plant. This is currently done mainly by experts, but this process is limited by the availability of experts, the complexity that may be solved by humans, and time. This paper introduces the Plant Model Inference Tool (PMIT) that infers deterministic context-free L-systems from an initial sequence of strings generated by the system using a genetic algorithm. PMIT is able to infer more complex systems than existing approaches. Indeed, while existing approaches are limited to L-systems with a total sum of $20$ combined symbols in the productions, PMIT can infer almost all L-systems tested where the total sum is $140$ symbols. This was validated using a testbed of 28 previously developed L-system models, in addition to models created artificially by bootstrapping larger models.
		
	\end{abstract}
	
	\section{Introduction}
	
	Lindenmayer systems (L-systems), introduced in \cite{lindenmayer_lsystems}, are a formal grammar system that produces self-similar patterns that appear frequently in nature, and especially in plants \cite{beauty}. L-systems produce strings that get rewritten over time in parallel. Certain symbols can be interpreted as instructions to create sequential images, which can be visually simulated by software such as the ``virtual laboratory'' (vlab) \cite{algorithmicbotany}. Such simulations are useful as they can incorporate different geometries \cite{beauty}, environmental factors \cite{drp_peach}, and mechanistic controls \cite{drp_auxin}, and are therefore of use to simulate and understand plants. L-systems often consist of small textual descriptions that require little storage compared to real imagery. Certainly also, they have a very low cost in currency, time, and labour, compared to actually growing a plant, as they produce the simulation very quickly even with low cost computers.
	
	An L-system is denoted by a tuple $G=(V,\omega,P)$, which consists of an alphabet $V$ (a finite set of allowed symbols), an axiom $\omega$ that is a word over $V$, and a finite set of productions, or rewriting rules, $P$. A deterministic context-free L-system or a D0L-system, has exactly one rule for each symbol in $V$ of the form $A \rightarrow x$, where $A  \in  V$ (called the predecessor) and $x$ is a word over V (called the successor). Words get rewritten according to a derivation relation, $\Rightarrow$, whereby $A_{1} \cdots A_{n} \Rightarrow x_{1} \cdots x_{n}$, where $A_{i}  \in V, x_{i}$ is a word, and $A_{i} \rightarrow x_{i}$ is in $P$, for each $i$, $1 \leq i \leq n$. Normally, one is concerned with derivations starting at the axiom, $\omega \Rightarrow x_{1} \Rightarrow x_{2} \Rightarrow \cdots \Rightarrow \omega_{n}$. Hence, a derivation involves replacing every character with its successor. For example, applying $A \to AB$ in a derivation involves replacing every $A$ in a word with $AB$. 
	
	Applying the rules iteratively produces repetitive patterns, and if the symbols are interpreted visually, they produce self-similar shapes. One common alphabet for visualization is the turtle graphics alphabet \cite{beauty}, so-called as it is imagined that each word produced by the L-system contains a sequence of instructions that causes a turtle to draw an image with a pen attached. The turtle has a state consisting of a position on a (usually) 3D grid and an angle, and the common symbols that cause the turtle to change states and draw are: $F$ (move forward with pen down), $f$ (move forward with pen up), $+$ (turn left), $-$ (turn right), $[$ (start a branch), $]$ (end a branch), $\&$ (pitch up), $\wedge$ (pitch down), \textbackslash (roll left), / (roll right), $\mid$ (turn around 180$^\circ$). Of special attention for branching models are the branching parantheses as [ causes the state to be pushed on a stack and ] causes the state to be popped and the turtle switches to it. It is assumed that the right hand side of rewriting rules have parantheses that are properly nested.  Additional symbols are added to the alphabet, such as $A$ and $B$, to represent the underlying growth mechanics. For example, a D0L-system can produce the image in Fig. \ref{fig:1} after 7 generations. Even more realistic 3D models may be produced with further extensions of L-systems, and a well-constructed model.
	
	\begin{figure}
		\centering
		\includegraphics[scale=0.4]{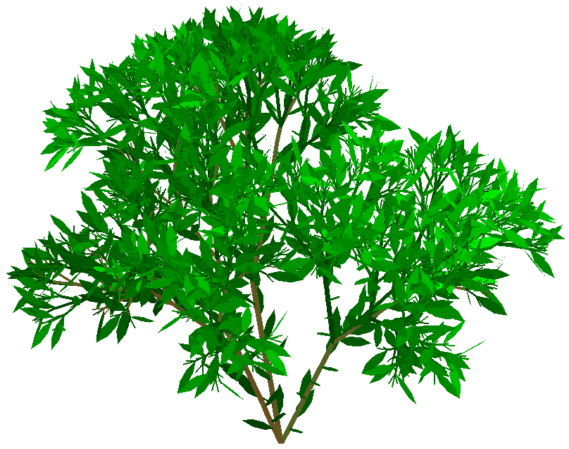}
		\caption{Fibonacci Bush after 7 generations \cite{beauty}.}
		\label{fig:1}
	\end{figure}
	
	A difficult challenge is to determine an L-system that can accurately simulate a plant. In practice, this often involves manual measurements over time, scientific knowledge, and is done by hand by experts \cite{beauty}. Although this approach has been successful, it does have notable drawbacks. Producing a system manually requires an expert that are in limited supply, and it does not scale to producing arbitrarily many (perhaps closely related) models. Furthermore, the more complex plant models require a priori knowledge of the underlying mechanics of the plant, which are increasingly difficult and time consuming to acquire. To address this, automated approaches have been proposed to overcome these limitations in two ways. One way is to use the automated approach as an aide for the expert to reduce the work load \cite{jacob_inferflowers,mock_wildwood}. Alternatively, the process may be fully automated to find an L-system that matches observed data \cite{nakano_inferD0Lerrorfree, runqiang_inferGA}. This approach has the potential to scale to constructing thousands of models, and also has the potential to expose biomechanics rather than requiring its knowledge beforehand.
	
	The ultimate goal of this direction of research is to automatically determine a model from a sequence of plant images over time. An intermediate step would be to infer the model from a sequence of strings used to draw the images. This is known as the inductive inference problem, defined as follows. Given a sequence of strings $\alpha = (\omega_{0}, \ldots, \omega_{n})$, determine a D0L-system (or that none exists) $G = (V,\omega, P)$ such that $\omega = \omega_{0} \Rightarrow \omega_{1} \Rightarrow \cdots \Rightarrow \omega_{n}$. That is, the problem is to determine some D0L-system that gives the sequence of input strings at the beginning of its derivation.
	
	This paper introduces the Plant Model Inference Tool (PMIT) that aims to be a fully automated approach to inductive inference of L-systems. Towards that goal, PMIT uses a genetic algorithm (GA) to search for an L-system to match the words produced. In general, GAs search solution spaces in accordance with the encoding scheme used for the problem, and to-date most existing approaches to L-system inference use similar encoding schemes. This paper presents a different encoding scheme and shows that it is more effective for inferring L-systems. Additionally, some logical rules are used as heuristics to shrink the solution space all of which are based in necessary conditions that must hold. Between these two techniques, it is determined that PMIT is able to infer L-systems where the sum of the right hand side of the rewriting rules is approximately $140$ symbols in length; whereas, other approaches are limited to about $20$ symbols. Moreover, the test bed used to test PMIT is significantly larger than previous approaches. Indeed, $28$ previously developed D0L-systems are used and for these systems that PMIT properly inferred, it did so in an average of $17.762$ seconds. Furthermore, additional (in some sense ``artificial'') models are created by combining the existing models where the combined length of the successors is longer than $140$ symbols (which PMIT does not solve within four hours), and then randomly removing ``F'' symbols until it can solve them. This work can be seen as a step towards the goal of 3D scanning a plant over time, and then converting the scanned imagery into symbols that describe how to draw it using a simulator, then inferring the L-system from the sequence of strings. 
	
	The remainder of this paper is structured as follows. Section 2 will describe some existing automated approaches for inferring L-systems. Section 3 will describe the logical rules used to shrink the solution space. Section 4 will discuss the genetic algorithms and the new encoding scheme used by PMIT. Section 5 will discuss the methodology used to evaluate PMIT and discuss the results. Finally, Section 6 concludes the work and discusses future directions.
	
	\section{Background}
	
	This section describes useful contextual and background information relevant to understanding this paper. It starts with describing some notation used throughout the paper. Since GA is used as the search mechanism for this work, it contains a brief description of genetic algorithms. Then, the section concludes with a discussion on some existing approaches to L-system inference. 
	
	\subsection{Notation}
	
	An alphabet is a finite set of symbols. Given an alphabet $V$, a word over $V$ is any sequence of letters written $a_{1}a_{2}\cdots a_{n}$, $a_{i}  \in V, 1 \leq i \leq n$. The set of all words over $V$ is denoted by $V^{*}$. Given a word $x \in V^{*}$, $|x|$ is the length of $x$, and $|x|_{A}$ is the number of $A$'s in $x$, where $A \in V$.
	
	Given an D0L-system $G=(V,\omega,P)$ as defined in Section $1$, the successor of $A$ is indicated by $succ(A)$. Given two words $x,y \in V^{*}$, then $x$ is a substring of $y$ if $y = uxv$, for some $u,v \in V^{*}$ and in this case $y$ is said to be a superstring of $x$. Also, $x$ is a prefix of $y$ if $y=xv$ for some $v$, and $x$ is a suffix of $y$ if $y=ux$ for some $u$. Additional notations are more easily understood in context and explained when appropriate.
	
	\subsection{Background on Genetic Algorithm}
	
	The GA is described as follows by Mitchell \cite{back_geneticalgorithm}. The GA is an optimization algorithm, based on evolutionary principles, used to efficiently search N-dimensional (usually) bounded spaces. In evolutionary biology, increasingly fit offspring are created over successive generations by intermixing the genes of parents. An encoding scheme is applied to convert a problem into a virtual genome consisting of N genes. Each gene is either a binary, integer, or real value and represents, in a problem specific way, an element of the solution to the problem. For example, in the travelling salesman problem (TSP), there are many types of possible encodings \cite{back_geneticalgorithm}; however, in this work two variations of \textit{value encoding} are used, which simply means the gene values are interpreted contextually for the problem. One common type of value encoding is \textit{literal encoding} \cite{back_geneticalgorithm}; for the TSP, each gene has an integer range from 1 to C, where C is the number of cities, and so each number represents the next city visited. Such an encoding is best when the number of options at each step of the problem are known and independent of earlier choices. For TSP, this encoding requires the graph of cities to be treated as fully connected and city connections that do not exist in the real-world are given extremely high weights to have the GA avoid choosing them. Another form of value encoding is a real value mapping to encode the genes onto the set of options available at each problem step \cite{back_geneticalgorithm}, e.g. a real value from $0$ and $1$, where the number of options available can depend on context. This encoding works best when the options at each step of the problem are unknown or dependent on prior choices. For TSP, this encoding allows the GA to only choose from connections that exist at each city without knowing how many choices might exist. If the salesman starts at node $A$ and there are two options, say $A \rightarrow B$ and $A \rightarrow C$, a gene can be encoded as a value between $0$ and $1$, where values up to $0.5$ represent the choice $A \rightarrow B$, and above $0.5$ represent $A \rightarrow C$. If there are three options from $B$ and four from $C$, then a gene encoded as a mapping can be decoded regardless of what traversal was made from $A$ as the three options from $B$ will each have a $33\%$ chance of being selected and each of the four options from $C$ will have a $25\%$ chance of selection, both based on a value from $0$ to $1$. 
	
	
	The GA functions by first creating an initial population ($P$) of random solutions. Each member of the population is assessed using a problem specific fitness function to assign a fitness value to each population member. Then the GA, controlled by certain parameters, performs a selection, crossover, mutation, and survival step until a termination condition is reached. In the selection step, a set of S pairs of genomes are selected from the population with odds in proportion to their fitness, i.e. preferring more fit genomes. A genome may be selected multiple times for different pairings, but the same two members may not be selected twice. During the crossover step, for each selected pair, a random selection of genes are copied between the two; thereby, producing two offspring. Each gene has a chance of being swapped equal to the control parameter \textit{crossover weight}. The mutation step takes each offspring and randomly changes zero or more genes to a random value with each gene having a chance of being mutated equal to the \textit{mutation weight}. Then each offspring is assigned a fitness value from the fitness function. The offspring are placed into the population and genomes are culled until the population is of size P again. Usually, the most fit members are kept (elite survival) but other survival operators exist that can prefer to preserve some genetic diversity. The termination condition may be based on such criteria as finding a solution with sufficient fitness, a pre-determined maximum number of generations, or a time limit.
	
	\subsection{Existing Automated Approaches to L-system Inference}
	
	Various approaches to L-system inference were surveyed in \cite{ben_naoum_surveryLsystems}. Here, we only mention certain works most closely related to ours. There are several different broad approaches towards the problem: building by hand \cite{beauty}, algebraic approaches \cite{doucet_algebra, nakano_inferD0Lerrorfree}, using logical rules \cite{nakano_inferD0Lerrorfree}, and search approaches \cite{runqiang_inferGA}. Since PMIT is a hybrid approach incorporating a search algorithm, GA, together with logical rules to reduce intractability by shrinking the search space, this section will examine some existing logic-based and search-based approaches.
	
	Inductive inference was studied theoretically (without implementation) by Doucet \cite{doucet_algebra}. He devised a method that uses solutions to Diophantine equations to, in many cases, find a D0L-system that starts by generating the input strings. A somewhat similar approach was implemented with a tool called LGIN \cite{nakano_inferD0Lerrorfree} that infers L-systems from a single observed string $\omega$. They recognize, using a similar method to Doucet \cite{doucet_algebra}, that the successors for the symbols in the alphabet can be expressed as a matrix, where a cell's value represents the number of symbols of each letter $a_{j}$ produced by each letter $a_{i}$. From the matrix, a set of equations are defined that relate, for $a_{i} \in V$, the number of symbols $a_{i}$ observed in $\omega$ to the linear combination of the production values in the matrix. From there, LGIN looks exhaustively at the successor combinations, extracted from $\omega$ that fulfill these equations. Since only a single string is used, LGIN does not guarantee to find a unique solution without being re-run. LGIN is limited to two symbol alphabets (common convention has that alphabet size does not count turtle graphic symbols), which is still described as ``immensely complicated'' \cite{nakano_inferD0Lerrorfree}; however, LGIN was evaluated on six variants of ``Fractal Plant'' \cite{beauty} and was very fast having a peak time to find the L-system of less than one second for 5 of the 6 variants, and four seconds for the remaining variant.
	
	Runqiang et al. \cite{runqiang_inferGA} propose to infer an L-system from an image using a GA. In their approach, they encode each gene to represent a symbol in the successors. The fitness function attempts to match the candidate system to the observed data. As an input, they use an image of the shape produced by an L-system to be inferred, as seen for example in Figure \ref{fig:1}. Their fitness function uses image processing to match the image produced by a candidate solution to the input image since their focus is on finding the proper angles and line lengths for drawing the image, in addition to the L-system itself. Their approach is limited to an alphabet size of $2$ symbols and a maximum total length of all successors of $14$. Their approach is 100\% successful for a variant of ``Fractal Plant'' \cite{beauty} with $|V| = 1$, and a 66\% success rate to find a variant of ``Fractal Plant'' \cite{beauty} with $|V| = 2$. Although they do not list any timings, their GA converged after a maximum of 97 generations, which suggests a short runtime.
	
	\section{PMIT Methodology for Logically Deducing Facts about Successors}\label{section:rules}
	
	In this section, the methodology that is used by PMIT to reduce the size of the solution space with heuristics --- all of which are based on necessary conditions for D0L-systems --- will be described. As all these conditions are mathematically true, this guarantees that a correct solution is in the remaining search space (if there is a D0L-system that can generate the input). The next section will then describe the GA used to search this space. Indeed, the success and efficiency of a search algorithm is generally tied to the size of the solution space. In PMIT, logical rules are used to reduce the dimensional bounds in two contexts. The first context is to determine a lower bound $\ell$ and upper bound $u$ on the number of each symbol $B$ produced by each symbol $A$ for each $A,B \in V$; that is, such that $\ell \le |succ(A)|_{B} \le u$, henceforth called growth of $B$ by $A$. A second context is a separate lower $\ell$ and upper bound $u$ on the length of each successor for each $A \in V$, such that $\ell \leq |succ(A)| \leq u$. The computation of growth partially relies on the computed lower and upper bounds of the lengths and vice versa, so all the rules are run in a loop until convergence (i.e. until the bounds stop improving). The program therefore initializes, for each $A \in V$, two programming variables $A_{min}$ and $A_{max}$ (whose values iteratively change by the program) such that $A_{min} \leq |succ(A)| \leq A_{max}$ and their values improve as the program runs. Similarly, for each $A,B \in V$, the program creates two variables $(A,B)_{min}$ and $(A,B)_{max}$, that change similarly such that $(A,B)_{min} \leq |succ(A)|_{B} \leq (A,B)_{max}$. 
	
	For this paper, it is assumed that if a turtle graphic symbol has an identity replacement rule (produces only itself, e.g. $\text{+} \rightarrow \text{+}$), then this is known in advance. Typically, the turtle graphic symbols do have an identity production; however, there are some instances where ``F'' may not (e.g. some of the variants of ``Fractal Plant'' \cite{beauty}). In such a case, ``F'' is treated as a non-turtle graphics symbol for the purposes of inferring the L-system although it is still visually interpreted as ``draw a line with the pen down''. Additionally, it is assumed that all successors are non-empty, since non-empty successors are used more commonly in practice when developing models \cite{beauty}. This implies that $A_{min}$ is initialized to 1 for each $A \in V$. For each turtle symbol $T \in V$, $T_{min} = T_{max} = 1$,$(T,T)_{min} = (T,T)_{max} = 1$ and $(T,A)_{min} = (T,A)_{max} = 0$ for every $A \in V, A \neq T$.
	
	The next two subsections describe the logical rules for deducing information about growth and successor length, starting with deducing growth.
	
	\subsection{Deducing Growth}
	
	Consider input $\alpha = (\omega_{0}, \ldots, \omega_{n})$, $\omega_{i} \in V^{*}$, $0 \leq i \leq n$ with alphabet $V$. Deduction of growth in PMIT is based on two mechanisms; the first being the determination of so-called \textit{successor fragments}, of which there are four types. 
	
	\begin{itemize}
		\item A word $\omega$ is an A-subword fragment if $\omega$ must be a subword of $succ(A)$.
		\item A word $\omega$ is an A-prefix fragment if $\omega$ must be a prefix of $succ(A)$.
		\item A word $\omega$ is an A-suffix fragment if $\omega$ must be a suffix of $succ(A)$.
		\item A word $\omega$ is an A-superstring fragment if $\omega$ must be a superstring of $succ(A)$.
	\end{itemize}
	
	As PMIT runs, it can determine additional successor fragments, which can help to deduce growth. Certain prefix and suffix fragments can be found for the first and last symbols in each input word by the following process.  Consider two words such that $\omega_{1} \Rightarrow \omega_{2}$. It is possible to scan $\omega_{1}$ from left to right until the first non-turtle graphics symbol is scanned (say, A, where the word scanned is $\alpha A$). Then, in $\omega_{2}$, PMIT skips over the graphical symbols in $\alpha$ (since each symbol in $\alpha$ has a known identity production), and the next $A_{min}$ symbols, $\beta$, (the current value of the lower bound for $|succ(A)|$) must be an A-prefix fragment. Furthermore, since branching symbols must be paired and balanced within a successor, if a $[$ symbol is met, the prefix fragment must also contain all symbols until a matching $]$ symbol is met. Similarly, an A-superstring fragment can be found by skipping $\alpha$ symbols, then taking the next $A_{max}$ symbols from $\omega_{2}$ (the upper bound on $|succ(A)|$). If a superstring fragment contains a $[$ symbol without the matching $]$ symbol, then it is reduced to the symbol before the unmatched $[$ symbol. Then, lower and upper bounds on the growth of $B$ by $A$ ($(A,B)_{min}$ and $(A,B)_{max}$) for each $B \in V$ can be found by counting the number of $B$ symbols in any prefix and superstring fragments respectively and changing them if the bounds are improved. For a suffix fragment, the process is identical except the scan goes from right to left starting at the end of $\omega_{1}$.
	
	\begin{example}
		Consider input strings $\omega_{1} = \text{+}\text{+}\text{+}A[\text{-}FF][\text{+}F]BF$ and $\omega_{2} = \text{+}\text{+}\text{+}+A[\text{-}FF][\text{-}FF][\text{+}F][\text{+}F]BFF$. Assume thus far PMIT has calculated that $A_{min} = 2$ and $A_{max} = 8$. Then, it can scan $\omega_{1}$ until $A$ is found and record that $\alpha = \text{+++}$. An $A$-prefix fragment is $+A$ as those are the first two ($A_{min}$) symbols of $\omega_{2}$ after skipping $\alpha$. An $A$-superstring fragment is $\text{+}A[\text{-}FF][$ as those are the first eight ($A_{max}$) symbols of $\omega_{2}$ after skipping $\alpha$, which can be reduced to $\text{+}A[\text{-}FF]$ due to the unbalanced $[$ symbol. By counting within the prefix fragment, lower bounds on the growth for $A$ are $(A,\text{+})_{min} := 1$ and $(A,A)_{min} := 1$, while upper bounds can be found from the superstring fragment to be $(A,\text{+})_{max} := 1, (A,\text{-})_{max} := 1$, $(A,A)_{max} := 1$,$(A,[)_{max} := 1$, $(A,])_{max} := 1$ and $(A,F)_{max} := 2$.
		\label{example:1}
	\end{example}	
	
	
	The second mechanism for deduction of growth is based on calculating the number of times each symbol $A \in V$ appears in word $\omega_{i}$ above the number implied from $\omega_{i-1}$ together with the current values of each lower bound $(B,A)_{min}$, for each $B \in V$. Formally, a programming variable for the \textit{accounted for growth} of a symbol $A \in V$ for $1 \leq i \leq n$, denoted as $G_{acc}(i,A)$ is:
	
	\begin{equation}
	G_{acc}(i,A) := \sum_{B \in V} (|\omega_{i-1}|_{B} \cdot (B,A)_{min}).
	\label{eq:1}
	\end{equation}
	
	\noindent The \textit{unaccounted for growth} for a symbol $A$, denoted as $G_{ua}(i,A)$, is computed as $G_{ua}(i,A) := |\omega_{i}|_{A} - G_{acc}(i,A)$. 
	
	The unaccounted for growth can improve the growth bounds. Then, $(B,A)_{max}$ is set (if it can be reduced) under the assumption that all unaccounted for $A$ symbols are produced by $B$ symbols. Furthermore, $(B,A)_{max}$ is set to be the lowest such value computed for any word from 1 to $n$, where $B$ occurs, as any of the $n$ words can be used to improve the maximum. And, $|succ(B)|_{A}$ must be less than or equal to $(B,A)_{min}$ plus the additional unaccounted for growth of $A$ divided by the number of $B$ symbols (if there is at least one) in the previous word, as computed by
	
	\begin{equation}
	(B,A)_{max} := \min_{1 \leq i \leq n, \atop |\omega_{i-1}|_{B} > 0} \Bigg((B,A)_{min} + \left \lfloor \tfrac{G_{ua}(i,A)} {|\omega_{i-1}|_{B}} \right \rfloor \Bigg).
	\label{eq:2}
	\end{equation}
	
	\noindent Indeed, the accounted for growth of $A$ is always updated whenever values of $(B,A)_{min}$ change, and the floor function is used since $|succ(B)|_{A}$ is a non-negative integer. For example, if $\omega_{i-1} = ABA$, $\omega_{i} = ABABBBABA$, $(A,A)_{min} = 1$, and $(B,A)_{min} = 0$, then the accounted for growth of $A$ in $\omega_{i}$ is computed by $G_{acc}(i,A) = (A,A)_{min} \cdot |\omega_{i-1}|_{A} + (B,A)_{min} \cdot |\omega_{i-1}|_{B} = 1 \cdot 2 + 0 \cdot 1 = 2$. This leaves two $A$'s in $\omega_{i}$ unaccounted for. An upper bound on the value of $|succ(A)|_{A}$ is set when the $A$'s in $\omega_{i-1}$ produce all of the unaccounted for growth in $\omega_{i}$. So $A$ produces its minimum ($(A,A)_{min} = 1$) plus the unaccounted for growth of $A$ in $\omega_{i}$ ($2$) divided by the number of $A$'s in $\omega_{i-1}$ ($|\omega_{i-1}|_{A} = 2$), hence $(A,A)_{max} := 2$. Similarly, $(B,A)_{max}$ is achieved when only $B$'s produce all unaccounted for growth of $A$; this sets $(B,A)_{max}$ to $(B,A)_{min} = 0$ plus the unaccounted for growth ($2$) divided by the number of $B$'s in $\omega_{i-1}$ ($1$), which is 2.
	
	Once $(B,A)_{max}$ has been determined for every $A,B  \in V$, the observed words are re-processed to compute possibly improved values for $(B,A)_{min}$. Indeed for each $(B,A)$, if $x := \sum_{C \in V \atop C \neq B} (C,A)_{max}$, and $x < |\omega_{i}|_{A}$, then this means that $|succ(B)|_{A}$ must be at least $\left \lceil \tfrac{|\omega_{i}|_{A} - x} {|\omega_{i-1}|_{B}} \right \rceil$, and then $(B,A)_{min}$ can be set to this value if its bound is improved. For example, if $\omega_{i-1}$ has 2 $A$'s and 1 $B$, and $\omega_{i}$ has 10 $A$'s, and $(A,A)_{max} = 4$, then at most two $A$'s produce eight $A$'s, thus one $B$ produces at least two $A$'s ($10$ total minus $8$ produced at most by $A$), and $(B,A)_{min}$ can be set to $2$.
	
	
	
	\subsection{Deducing Successor Length}
	
	The deduction of $A_{min}$ and $A_{max}$ are found from two logical rules, one involving the sum of the minimum and maximum growth over all variables, and one by exploiting a technical mathematical property. The first rule simply states that $A_{min}$ is at least the sum of $(A,B)_{min}$ for every $B \in V$ and similarly $A_{max}$ is at most the sum of $(A,B)_{max}$ for every $B \in V$. The second rule is trickier but can sometimes improve the upper and lower bounds for $A_{max}$ and $A_{min}$ for $A \in V$. This takes place in two steps. First, the maximum number of symbols that can be produced by $A$ in $\omega_{i}$ is computed by: $x := |\omega_{i}| - \sum_{B \in V, B \neq A} (B_{min} \cdot |\omega_{i-1}|_{B})$. If $|\omega_{i-1}|_{A}  > 0$, let:
	
	\begin{equation}
	A^{i}_{max} := \left \lfloor \tfrac{x} {|\omega_{i-1}|_{A}} \right \rfloor
	\label{equation:x}
	\end{equation}
	
	\noindent if its value is improved. It is immediate that $A_{max}$ can be set to $\min_{1\leq i \leq n, \atop |\omega_{i-1}|_{A}> 0} A^{i}_{max}$, if its value is improved.
	
	For the second step, now that these $A^{i}_{max}$ values been calculated, to possibly improve the $A_{max}$ and $A_{min}$ values further, some intermediate values are required. Let $Y^{i} \in V$, $1 \leq i \leq n$ be such that $Y^{i}$ occurs the least frequently in $\omega_{i-1}$ with at least one copy. The current value of $Y^{i}_{max}$ will be examined as computed by Equation \ref{equation:x}; note $Y^{1}$, \ldots, $Y^{n}$ can be different. Let $V^{i}_{max} := Y^{i}_{max} + \sum_{B \in V, \atop B \neq Y^{i}} B_{min}$. 
	
	The computation of a total value, $V^{i}_{max}$, can allow refinement of the upper bound for each successor. Then, $A_{max}$ may be improved by assuming all other symbols produce their minimum and subtracting from the total maximum values.  Mathematically this is expressed as:
	
	\begin{equation}
	A_{max} := V^{i}_{max} - \sum_{B \in V, \atop B \neq A} B_{min}
	\end{equation}
	
	\noindent for $1 \leq i \leq n$, if $A$ occurs in $\omega_{i-1}$, and if the new value is smaller, which has the effect of the minimum over all $i$, $1 \leq i \leq n$. Although it is not immediately obvious that this formula is an upper bound on $|succ(A)|$, a proof is provided.


	\begin{claim}
		$|succ(A)| \leq V^{i}_{max} - \sum_{B \in V, \atop B \neq A} B_{min}$.
	\end{claim}
	
	\begin{proof}	
		From the definition of $V^{i}_{max}$, $V^{i}_{max}  - \sum_{B \in V, \atop B \neq A} B_{min} = Y^{i}_{max} + \sum_{B \in V, \atop B \neq Y^{i}} B_{min} - \sum_{B \in V, \atop B \neq A} B_{min}$, and subtracting identical terms from the two summations gives $Y^{i}_{max} + A_{min} - Y^{i}_{min}$. Assume by contradiction that $|succ(A)|$ is greater than this formula. By Equation \ref{equation:x}, $Y^{i}_{max}$ is the maximum possible integer value such that $\sum_{B \in V, \atop B \neq Y^{i}} (B_{min} \cdot |\omega_{i-1}|_{B}) + Y^{i}_{max} \cdot |\omega_{i-1}|_{Y^{i}} \leq |\omega_{i}|$. Notice that the left hand side is equal to $\sum_{B \in V} (B_{min} \cdot |\omega_{i-1}|_{B}) + (Y^{i}_{max} - Y^{i}_{min}) \cdot |\omega_{i-1}|_{Y^{i}} \leq |\omega_{i}|$ and $(Y^{i}_{max} - Y^{i}_{min})$ is the largest integer possible such that this happens. But the summation gives all letters that can be derived from the $B_{min}$ values. Since $|succ(A)| - A_{min} > Y^{i}_{max} - Y^{i}_{min}$ and $|\omega_{i-1}|_{A} \geq |\omega_{i-1}|_{Y^{i}}$ (since $Y^{i}$ is the least frequently occurring symbol), this implies $\sum_{B \in V} (B_{min} \cdot |\omega_{i-1}|_{B}) + (|succ(A)| - A_{min}) \cdot |\omega_{i-1}|_{A} \leq |\omega_{i}|$. This contradicts the maximality of $Y^{i}_{max} - Y^{i}_{min}$, hence the inequality holds.
		\qed
	\end{proof}

	Thus, $A_{max}$ can be set in this fashion. Similarly, $A_{min}$ can be set by taking $Y^{i}$ that occurs most frequently.
	
	\begin{example}
		Consider the inputs $\omega_{1} = \text{+}AB\text{+}A$ and $\omega_{2} = \text{+}\text{+}ABABBB\text{+}\text{+}ABA$. Assume that $A_{min} = 2$, $B_{min} = 1$, then $B^{1}_{max}$ is computed in Equation \ref{equation:x} by taking $|\omega_{2}| = 13 - (A_{min} \cdot |\omega_{1}|_{A} + (\text{+}_{min}) \cdot |\omega_{1}|_{+} = 13 - (2 \cdot 2 + 1 \cdot 2) = 7$, and dividing by $|\omega_{1}|_{B}$ gives $1$. Hence, $B^{1}_{max} := 7$ and $B$ occurs least frequently in $\omega_{1}$. Thus $Y^{1} = B$, $Y^{1}_{max} := 7$, and $V^{1}_{max} := 7 + A_{min} + (\text{+}_{min}) = 10$. Then $A_{max} := V_{max} - B_{min} - (\text{+}_{min}) = 8$.
		\label{example:2}
	\end{example}

	\section{Encoding for the L-system Inference Problem}
	
	In this section, the GA and encoding used by PMIT is described and contrasted with previous approaches.
	
	\subsection{Parameter Optimization}
	
	The efficient search of a GA is controlled, in part, by the settings of the control parameters: population size, crossover weight, and mutation weight. Optimal parameter settings have been shown to be problem specific \cite{hyperparameter}. The process of finding the optimal control parameter settings is called \textit{hyperparameter search}. Although the parameter settings can be tuned by hand, the advantage to searching is that it can be more complete and free of experimental error. The drawback is that hyperparameter search algorithms usually have a control parameter of their own; however, it has been found that Random Search (RS) provides good parameter optimization \cite{hyperparameter}. RS works by trying different parameter sets called trials. The one that produces the best result is considered the new best. This process continues until none of the trials produces a better result than the current best parameter set. Sixteen trials were found to provide good results for a wide variety of problems and so this was used as the trial limit \cite{hyperparameter}. The parameters settings for the GA were set as follows, as suggested by Grefenstette \cite{grefenstette_optimalGAparameter}. For population size, the settings were bound from 10 to 125 in increments of 5. For crossover weight, the settings were bound from 0.6 to 0.95 in increments of 0.05. For mutation weight, the settings were bound from 0.01 to 0.20 in increments of 0.01 with the additional values of 0.001 and 0.0001 permitted. For PMIT, the optimal parameter settings were shown to be 100 for population size, 0.85 crossover weight, and 0.10 for mutation weight. These parameters were used for all further analysis.
	
	\subsection{Fitness Function and Termination Conditions}
	
	The fitness function for PMIT compares the symbols in the observed data to the symbols in the words produced by the candidate solution position by position. An error is counted if the symbols do not match or if the candidate solution is too long or short. For example, when comparing the words $ABAAB$ to $ABBA$, there are three errors. The third and fourth symbols do not match, and a fifth symbol is expected but does not exist. This gives a fitness $F$ of $3 / 5$. If the candidate solution produces more than double the number of symbols expected, it is not evaluated and assigned an extremely high fitness so that it will not pass the survival step. Since errors early on in the input words $\omega_{0}, \ldots ,\omega_{n}$ will cause errors later, a word $\omega_{i}$, is only assessed if there are no errors for each preceding word, and 1.0 is added to $F$ for each unevaluated word. This strongly encourages the GA to find solutions that incrementally produce more and more of the observed words.
	
	PMIT uses three termination conditions to determine when to stop running. First, PMIT stops if a solution is found with a fitness of 0.0 as such a solution perfectly describes the observed data. Second, PMIT stops after 4 hours of execution if no solution has been found. Third, PMIT stops when the population has converged and can no longer find better solutions. This is done by recording the current generation whenever a new best solution is found as $Gen_{best}$. If after an additional $Gen_{best}$ generations, no better solutions are found, then PMIT terminates. To prevent PMIT from terminating early due to random chance, PMIT must perform at least 1,000 generations for the third condition only. This third condition is added to prevent the GA from becoming a random search post-convergence and finding an L-system by chance skewing the results.
	
	\subsection{Encoding Scheme and Genomic Structure}
	
	The encoding scheme used most commonly in literature (e.g., \cite{runqiang_inferGA,mock_wildwood}) is to have a gene represent each possible symbol in a successor. The number of genes for the approaches in literature varies due the specific method they use to decode the genome into an L-system, although they are approximately the total length of all successors combined. With this approach, each gene represents a variable from $V$; internally, this is an integer from $1$ to $|V|$ associated to the variables. However, in some approaches (and PMIT) the decoding step needs to account for the possibility that a particular symbol in a successor \textit{does not exist} (represented by $\oslash$). When the possibility of an $\oslash$ exists, such genes have a range from $1$ to $|V|+1$. As an example, assume $V = \{A,B\}$ and $A_{min} = 2, A_{max} = 3, B_{min} = 1, B_{max} = 3$. For $A$, it is certain to have at least two symbols in the successor and the third may or may not exist. So, the first three genes represent the symbols in $succ(A)$, where the first two genes have each possible values from $\{A,B\}$ and the third gene has three possible values $\{A,B,\oslash\}$.
	
	Next the improvements made to the genomic structure defined by the basic encoding scheme will be described. Although they are discussed separately for ease of comprehension, all the improvements are used together.
	
	As mentioned, PMIT uses facts deduced about the possible successors to create a genomic structure for a GA. Consider the case where $V = \{A,B\}$, $A_{min} = 1$, and $A_{max} = 3$, then $succ(A)$ can be expressed as the genomic structure of $\{A,B\},\{A,B,\oslash\},\{A,B,\oslash\}$. If it is determined that $A$ has an A-prefix of $B$, then this changes the genomic structure to $\{B\},\{A,B,\oslash\},\{A,B,\oslash\}$ since it is certain that the first symbol in $succ(A)$ is $B$. Essentially, this eliminates the need for the first gene. This is similar for an A-suffix. 
	
	The second improvement to the basic encoding scheme further reduces the solution space by using growth facts to eliminate some impossible solutions. For example, if it is deduced that $(A,B)_{min} = 1$, $(A,B)_{max} = 3$, and $A_{min} = 4$, then there is no need to consider the solution $A \to BBBB$ because there are at most three $B$'s. With this in mind, when building the successor, PMIT first places the symbols known to be in $succ(A)$. To do this, PMIT uses a mapping with real-value between $0$ and $1$. For each successor, $\sum_{A,B \in V} (A,B)_{min}$ genes are created with a real value range between $0$ and $1$. Since these symbols are known to exist, the mapping is used to select an unused position within the successor. After these symbols are placed, if any additional symbols are needed up to the value of $A_{max}$, then PMIT selects the remainder by using $\oslash$, while ensuring that $(A,B)_{max}$ is not violated for any $A,B \in V$.
	
	For example, if $A_{max} = 4$, then the successor can be thought of as $A \rightarrow \_ \thinspace \_ \thinspace \_ \thinspace \_$ where ``$\_$'' represents an unused position. So if $V = \{A,B\}$ and $(A,A)_{min} = 1$ and $(A,B)_{min} = 1$, then the first gene will be an $A$ (placed at one of four positions) and the second gene a $B$. Suppose the first gene value is such that the second position is selected between $0.25$ and $0.5$, then the successor may be thought of as $A \rightarrow \_ \thinspace A \thinspace \_ \thinspace \_$. The second gene can then only select from the first, third, or fourth position, hence the use of a mapping since the available options are dependent on the first choice made. So, if the second gene selects the third position for the $B$ between $0.33$ and $0.67$, then the successor may be thought of as $A \rightarrow \_ \thinspace A \thinspace B \thinspace \_$. The remaining two genes represent the symbol to be placed in the unused positions; however, these symbols allow for the possibility of an $\oslash$ and without violating any value of $(A,B)_{max}$. If the only consideration were the $\oslash$ the basic encoding could be used; however, these genes also require a real value mapping to enforce the second constraint. To continue the example, if $(A,A)_{max} = 2$ and $(A,B)_{max} = 2$, then if the third gene (regardless of encoding scheme) is an $A$ then the successor is $A \rightarrow A \thinspace A \thinspace B \thinspace \_$, which implies that no further $A$'s can be selected. Thus, to decode the fourth gene depends on the choice of the third gene, hence the use of a real value mapped encoding.
	
	\subsection{Independent Sub-Problems}
	
	In other existing approaches (e.g., \cite{nakano_inferD0Lerrorfree,runqiang_inferGA,mock_wildwood}), the successors are found all in one step; however, this is not necessary. Since, non-graphical symbols can only be produced by non-graphical symbols, it is possible to, at first, ignore the graphical symbols, in essence creating a smaller alphabet $V_{im}$. It is possible to search for the successors over $V_{im}^{*}$, which is a much simpler problem. For example, if $A \to F[\text{+}F]B$ and $B \to F[\text{-}F]A$, then with $V_{im} = \{A,B\}$ it is only necessary to find $A \to B$ and $B \to A$. Then each graphical symbol can be solved for individually. So to continue the example above, the second step might be to add $+$ to $V_{im}$ and find $A \to +B$ and $B \to A$. Solving these smaller problems is more efficient as the individual search spaces are smaller and when summed are smaller than the solution space when trying to find the full successor in one step. To infer an L-system in this way requires two encoding schemes, one for solving the first sub-problem and then an encoding scheme for solving each graphical symbol sub-problem. All of the logical rules are performed for every sub-problem, so when referring to $A_{max}$, for example, this is within the context of the current sub-problem. In addition, to split the problem up, the prefix and suffix fragments are used to partially solve the problem. For the first sub-problem, the encoding works exactly as already described in the preceding two subsections. 
	
	For the remaining sub-problems, the solution from the previous sub-problem is treated as a subword fragment for symbols that have no prefix/suffix fragment. Conceptually, the symbols are being inserted, before, after or in the middle of the subword fragments defined by the previous sub-problem's solution. To continue the example above, if $A \to AB$ is the solution to the first sub-problem for symbol $A$, $V_{im} = \{A,B,\text{+}\}$, $A_{min} = 2$, and $A_{max} = 5$, then the successor can be represented by the genomic structure:
	
	\begin{center}
		$\{\text{+},\oslash\},\{\text{+},\oslash\},\{\text{+},\oslash\},\{A\},\{\text{+},\oslash\},\{\text{+},\oslash\},\{\text{+},\oslash\},\{B\},\{\text{+},\oslash\},\{\text{+},\oslash\},\{\text{+},\oslash\}$.
	\end{center}
	
	Although this may appear at first like it needs nine genes to represent all the different possibilities, since $A_{max} = 5$ and two symbols already exist in $succ(A)$, there can be at most three $+$ symbols. In other words, in the genomic structure at least six of the values are $\oslash$, so alternatively it can be thought of as $A \rightarrow \_ \thinspace \_ \thinspace \_ \thinspace A \_ \thinspace \_ \thinspace \_ \thinspace B \_ \thinspace \_ \thinspace \_ \thinspace$, and at most three of the nine positions need to be selected, while also allowing for a $\oslash$ to be placed in lieu of a $+$. To do this only three genes are needed using real value mapped encoding, where half of the range corresponds to placing a $\oslash$ in a position, and the remainder is mapped to putting a $+$ in a position as previously described.
	
	For symbols with a prefix/suffix fragment, they are modified to remove any symbols not previously found except the current sub-problem's symbol. For example, if the current symbol is $+$, and the A-prefix is $\text{+}A\text{-}\text{+}$ then the A-prefix is modified to $\text{+}A\text{+}$. With such a prefix, no further $+$ symbols can be added before or in the middle of the $\text{+}A\text{+}$ as this is not possible, and similarly for a suffix fragment no symbols may be added after or in the middle of it.
	
	\section{Evaluation and Results}
	
	\subsection{Data Set}
	
	To evaluate PMIT's ability to infer D0L-systems, ten fractals, plus the six plant-like fractal variants inferred by the existing program LGIN \cite{beauty,nakano_inferD0Lerrorfree}, and twelve other biological models were selected from the vlab online repository \cite{algorithmicbotany}. The biological models consist of ten algaes, apple twig with blossoms, and a ``Fibonacci Bush'', which was found to be aesthetically realistic \cite{beauty}. The dataset compares favourably to similar studies where only some variants of one or two models are considered \cite{nakano_inferD0Lerrorfree,runqiang_inferGA}. The data set is also of greater complexity by considering models with alphabets from between $2$ to $31$ symbols compared to two symbol alphabets \cite{nakano_inferD0Lerrorfree,runqiang_inferGA}. In examining the data set, it was observed that there were gaps both in terms of successor lengths and alphabet size. To more fully evaluate PMIT, additional L-systems were created by bootstrapping. This was done by taking L-systems from the test suite above, then combining successors from multiple L-systems to create L-systems with every combination of alphabet size from $3$ to $25$ in increments of 2 and longest successor length from $5$ to $25$ in increments of 5. To get successors of the proper length some ``F'' symbols were trimmed from longer successors. This set of ``fake'' L-systems are called \textit{generated L-systems}.
	
	\subsection{Performance Metrics}
	
	Two performance metrics are used to measure how well PMIT can infer D0L-systems. The first metric is \textit{success rate} (SR) which is defined as the percentage of times PMIT can find any L-system that describes the observed data. The second metric is \textit{mean time to solve} (MTTS) measured to the millisecond level since some models solve in sub-second time. Time was measured using a single core of an Intel 4770 @ 3.4 GHz with 12 GB of RAM on Windows 10. PMIT is only allowed to execute for at most $4$ hours (14400 seconds) and reaching this limit is considered a failure. It is reasoned that if PMIT cannot solve a D0L-system in this amount of time, then it not a particularly practical tool. These metrics are consistent with those found in literature \cite{nakano_inferD0Lerrorfree, runqiang_inferGA}.
	
	\subsection{Results}
	
	Three programs were evaluated. The first is the full PMIT program (implemented in C++ using Windows 10), the second is a simpler restriction of PMIT that uses a brute force algorithm implemented without using the GA or logical rules, and the existing program LGIN. Results are shown in Table \ref{table:results1}. No SR is shown for LGIN as it is not explicitly stated; however, it is implied to be 100\% \cite{nakano_inferD0Lerrorfree} for all rows where a time is written. The variants used by LGIN \cite{beauty, nakano_inferD0Lerrorfree} are marked as ``Fractal Plant $v1\ldots v6$''. The success rates are all either 0\% or 100\% for the non-generated models, indicating that a problem is either solved or not. In general, PMIT is fairly successful at solving the fractals, and the ``Fractal Plant'' variants, and was also able infer \textit{Ditria reptans}. It was observed that PMIT was able to solve the other plant models up to the point of inferring the $F$ and $f$ symbols, as indicated in the ``Infer Growth'' column. For example, PMIT inferred for \textit{Aphanocladia} that $A \to BA$, $B \to U[-C]UU[+/C/]U$; however, it was not able to then infer $C \rightarrow FFfFFfFFfFF[\text{-}F^{4}]fFFfFF[\text{+}F^{3}]fFFfFF[\text{-}FF]fFFf$. From a practical perspective, this is useful as the difficult part, for a human, is determining the growth mechanisms, more so than the lines represented by the $F$ and $f$ symbols. Therefore, PMIT is a useful aide to human experts even when it cannot infer the complete L-system.
	
	For the generated models, a success rate is not reported, instead a generated model is considered \textit{solved} if there is a 100\% success rate and \textit{unsolved} otherwise. Figure \ref{fig:2} gives one point for every L-system, either generated or not, that PMIT attempted to solve. They are represented by a diamond it was inferred with 100\% SR and a cross otherwise. It is evident that the figure shows a region described by alphabet size and longest successor length that PMIT can reliably solve. PMIT can infer L-systems with $|V|=17$, if the successors are short (5) and can infer fairly long successors (25) when $|V|=3$. Computing the \textit{sum of successor words} $\sum_{A \in V} |succ(A)|$, then PMIT is able to infer L-systems where such a sum is less than $140$, which compares favorably to approaches in literature where the sum is at most $20$.
	
	Overall, in terms of MTTS, PMIT is generally slower than LGIN \cite{nakano_inferD0Lerrorfree} for $|V|=2$; however, PMIT can reliably infer L-systems with larger alphabet sizes and successor lengths and still does so with a practical average of $17.762$ seconds. Finally, in comparison to the brute force algorithm with the basic encoding scheme it can be seen that the bounds provided by the logical rules, plus the improved encoding scheme and the use of a GA provide considerable improvement as the average MTTS is $621.998$ seconds.
	
	\begin{table}[h!]
		\centering
		\begin{tabular}{|c||c|c|c||c|c||c|}
			\hline
			\multirow{2}{*}{Model} & \multicolumn{3}{c||}{PMIT} & \multicolumn{2}{c|}{Brute Force} & LGIN \cite{nakano_inferD0Lerrorfree} \\ \cline{2-7}
			& SR & MTTS (s) & Infer Growth & SR & MTTS (s) & MTTS (s) \\ \hline
			Algae \cite{beauty}& 100\% & 0.001 & n/a & 100\% & 0.001 & - \\ \hline
			Cantor Dust \cite{beauty}& 100\% & 0.001 & n/a & 100\% & 0.001 & - \\ \hline
			Dragon Curve \cite{beauty}& 100\% & 0.909 & n/a & 100\% & 4.181 & - \\ \hline
			E-Curve \cite{beauty}& 0\% & 14400 & Yes & 0\% & 14400 & - \\ \hline
			Fractal Plant v1 \cite{beauty,nakano_inferD0Lerrorfree} & 100\% & 33.680 & n/a & 100\% & 163.498 & 2.834 \\ \hline
			Fractal Plant v2 \cite{beauty,nakano_inferD0Lerrorfree} & 100\% & 0.021 & n/a & 100\% & 5.019 & 0.078 \\ \hline
			Fractal Plant v3 \cite{beauty,nakano_inferD0Lerrorfree} & 100\% & 0.023 & n/a & 100\% & 5.290 & 0.120 \\ \hline
			Fractal Plant v4 \cite{beauty,nakano_inferD0Lerrorfree} & 100\% & 0.042 & n/a & 100\% & 6.571 & 0.414 \\ \hline
			Fractal Plant v5 \cite{beauty,nakano_inferD0Lerrorfree} & 100\% & 34.952 & n/a & 100\% & 171.003 & 0.406 \\ \hline
			Fractal Plant v6 \cite{beauty,nakano_inferD0Lerrorfree} & 100\% & 31.107 & n/a & 100\% & 174.976 & 0.397 \\ \hline
			Gosper Curve \cite{beauty}& 100\% & 71.354 & n/a& 100\% & 921.911 & - \\ \hline
			Koch Curve \cite{beauty}& 100\% & 0.003 & n/a& 100\% & 0.023 & - \\ \hline
			Peano \cite{beauty}& 0\% & 14400 & Yes & 0\% & 14400 & - \\ \hline
			Pythagoras Tree \cite{beauty}& 100\% & 0.041 & n/a & 100\% & 2.894 & - \\ \hline
			Sierpenski Triangle v1 \cite{beauty}& 100\% & 2.628 & n/a & 100\% & 267.629 & - \\ \hline
			Sierpenski Triangle v2 \cite{beauty}& 100\% & 0.086 & n/a & 100\% & 128.043 & - \\ \hline
			\textit{Aphanocladia} \cite{algorithmicbotany}& 0\% & 54.044 & Yes & 0\% & 14400 & - \\ \hline
			\textit{Dipterosiphonia} v1 \cite{algorithmicbotany}& 0\% & 14400 & No & 0\% & 14400 & - \\ \hline
			\textit{Dipterosiphonia} v2 \cite{algorithmicbotany}& 0\% & 14400 & Yes & 0\% & 14400 & - \\ \hline
			\textit{Ditira Reptans} \cite{algorithmicbotany}& 100\% & 73.821 & n/a & 100\% & 6856.943 & - \\ \hline
			\textit{Ditira Zonaricola} \cite{algorithmicbotany}& 0\% & 74.006 & Yes & 0\% & 14400 & - \\ \hline
			\textit{Herpopteros} \cite{algorithmicbotany}& 0\% & 81.530 & Yes & 0\% & 14400 & - \\ \hline
			\textit{Herposiphonia} \cite{algorithmicbotany}& 0\% & 298.114 & Yes & 0\% & 14400 & - \\ \hline
			\textit{Metamorphe} \cite{algorithmicbotany}& 0\% & 14400 & Yes & 0\% & 14400 & - \\ \hline
			\textit{Pterocladellium} \cite{algorithmicbotany}& 0\% & 14400 & No & 0\% & 14400 & - \\ \hline
			\textit{Tenuissimum} \cite{algorithmicbotany}& 0\% & 14400 & No & 0\% & 14400 & - \\ \hline
			Apple Twig \cite{algorithmicbotany}& 0\% & 14400 & No & 0\% & 14400 & - \\ \hline
			Fibonacci Bush \cite{algorithmicbotany}& 0\% & 14400 & Yes & 0\% & 14400 & - \\ \hline
		\end{tabular}
		\caption{Results for PMIT, Brute Force, and LGIN \cite{nakano_inferD0Lerrorfree}, on previously developed L-system models.}
		\label{table:results1}
	\end{table}
	
	\begin{figure}[h!]
		\centering
		\includegraphics[scale=0.65]{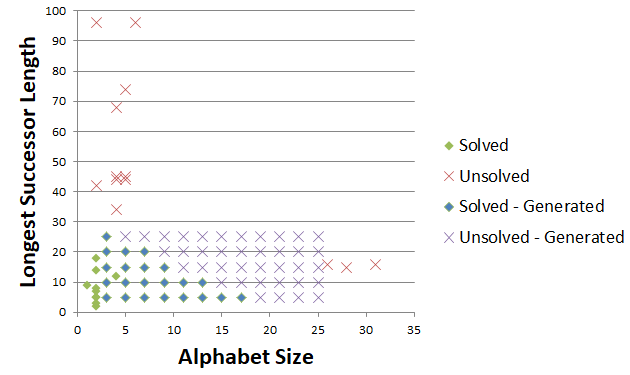}
		\caption{L-Systems Solved with 100\% SR by Alphabet Size and Longest Successor Length}
		\label{fig:2}
	\end{figure}
	
	\section{Conclusions}
	
	This paper has introduced the Plant Model Inference Tool (PMIT) as a hybrid approach, combining GA and logical rules, to infer deterministic context-free L-systems (D0L-systems). Using a search algorithm allows PMIT to infer a D0L-system to the observed data of the model, making PMIT a general solution to different models and domains. The drawback to searching for an L-system is intractability. PMIT decreases the search space by using logical rules to deduce facts about the unknown D0L-system to more tightly bind the search space. The limits of PMIT are defined by an alphabet size up to 17 and individual successor lengths up to 25. Overall, PMIT can infer systems where the sum of the successor lengths is less than or equal to $140$ symbols. This compares favourably to existing approaches that are limited to one or two symbol alphabets, individual successor lengths less than or $18$, and a total successor length less than or equal to 20 \cite{nakano_inferD0Lerrorfree,runqiang_inferGA}, especially considering the 2 symbol alphabet has been described as ``immensely complicated'' \cite{nakano_inferD0Lerrorfree}. Although, PMIT is generally slower than existing approaches when the alphabet is small (2) \cite{nakano_inferD0Lerrorfree,runqiang_inferGA}, it still infers these L-systems in less than $18$ seconds, which is fast enough to be practical.
	
	PMIT is able to infer D0L-systems fully automated under the right circumstances of alphabet size and successor length, so for such systems the potential impact is very high. PMIT is very suitable as an aide to experts by inferring the growth pattern for an unknown model as it can do so under circumstances known to exist in the real-world \cite{beauty}. Once the growth pattern is inferred, it is possible for a human to trace out the corresponding drawing pattern from the model and devise the proper turtle graphics symbols to complete the L-system.
	
	For future work, the main focus will be on improving PMIT's ability to properly infer the drawing pattern. It may be that searching for the drawing pattern can be performed using image processing techniques or, since this research shows that the inference problem can be sub-divided, alternative searching techniques can be used for the drawing patterns. Finally, additional mechanisms will be investigated to further extend the limits of alphabet size and successor length for PMIT.
	
	\bibliography{mybib}{}
	\bibliographystyle{splncs03}
	
\end{document}